\documentclass[letterpaper, 10pt, conference]{ieeeconf}      
\usepackage{amsfonts}
\usepackage{amsmath,amssymb}
\usepackage{cite}
\usepackage{graphicx}
\usepackage{hyperref}
\usepackage{wrapfig}
\usepackage{times}

\newtheorem{theorem}{Theorem}

\newtheorem{definition}[theorem]{Definition}

\newtheorem{lemma}[theorem]{Lemma}

\newtheorem{proposition}[theorem]{Proposition}

\newcommand{\Real}{\mathbb R}

\IEEEoverridecommandlockouts                              
\usepackage{times} 

\title{\LARGE \bf
Exponentially Fast Parameter Estimation in Networks\\
Using Distributed Dual Averaging$^\dagger$}

\author{Shahin Shahrampour$^\ddagger$ and Ali Jadbabaie$^\ddagger$
\thanks{$^\dagger$This work was supported by AFOSR MURI CHASE, and ONR BRC Program  on Decentralized, Online Optimization.}
\thanks{$^\ddagger$The authors are with the Department of Electrical and Systems Engineering
and General Robotics, Automation, Sensing and Perception (GRASP)
Laboratory, University of Pennsylvania, Philadelphia, PA 19104-6228 USA.
        {\tt\small \{shahin,jadbabai\}@seas.upenn.edu}}%
}

\begin{document}

\maketitle
\thispagestyle{empty}
\pagestyle{empty}

\begin{abstract}

In this paper we present an optimization-based view of distributed parameter
estimation and observational social learning in networks. Agents receive a
sequence of random, independent and identically distributed (i.i.d.) signals, each of which individually 
may not be informative about the underlying true state, but the
signals together are globally informative enough to make the true state identifiable. Using an
optimization-based characterization of Bayesian learning as proximal
stochastic gradient descent (with Kullback-Leibler divergence from a prior as a
proximal function), we show how to efficiently use  a distributed, online
variant of Nesterov's dual averaging method to solve the estimation with purely local information. 
When the true state is globally identifiable, and the network is connected, we prove that
agents eventually learn the true parameter using a randomized gossip scheme. We demonstrate that with high probability the convergence is exponentially fast with a rate dependent on the KL divergence of observations under the true state from observations under the second likeliest state. Furthermore, our work also highlights the possibility of learning under continuous adaptation of network which is a consequence of employing constant, unit stepsize for the algorithm.
\end{abstract}

\section{Introduction}

Distributed estimation, detection, and observational social learning  has been an intense focus of research over the past 3 decades~\cite{BorVar82,Tsit88,tsitsiklis1993decentralized, mossel2010efficient,KhanJad2010, kar2012distributed,rad2010distributed, jadbabaie2012non,dekel2012optimal},
with applications ranging from sensor networks to social and economic networks. In these scenarios,  agents in a network need  to learn the value of a parameter, that might represent a state or decision (often called the state of the world),  but each individual agent lacks the necessary information to estimate the state on its own. Instead, the global spread of information in the network provides agents with adequate data for recovering the true state and as a result, agents iteratively exchange information with their neighbors. In distributed sensor and robotic networks, agents use local diffusion to augment their imperfect observations with information from their neighbors~\cite{tsitsiklis1984problems,jadbabaie2003coordination,kar2012distributed,EgerMesBook, bullo2009distributed,Olfati05}.

On the other hand, recent developments in distributed optimization have led to many advances and interesting decentralized algorithms, generalizing these results and at the same time opening new venues for development of principled  distributed estimation algorithms.
Examples of such papers  include the works of researchers such as Nedi\'c and Ozdaglar~\cite{nedic2009distributed},  Lobel and Ozdaglar~\cite{lobel2008distributed}, Ram, Nedi\'c and Veeravalli~\cite{ram2010distributed} and Nedi\'c, Olshevsky, Ozdaglar and Tsitsiklis~\cite{nedic20092distributed}, Lopes and Sayed~\cite{lopes2007incremental},
and more recently the results of  Duchi, Agarwal and Wainwright~\cite{duchi2010dual}. Of particular importance to the work in this paper is the work of Duchi {\it et al.} in~\cite{duchi2010dual} where the authors  develop a distributed method based on dual averaging of subgradients. Using proper diminishing stepsize rule, their algorithm converges to the optimal solution in deterministic network change as well as stochastic.

The goal of this paper is to provide an optimization-based formulation of parameter estimation and social learning and develop a link between the two. Our motivation for the current study is the recent results of~\cite{rad2010distributed} and~\cite{jadbabaie2012non} in which the authors develop non-Bayesian learning schemes to circumvent the complexities associated with fully Bayesian estimation~\cite{BorVar82}, as well as the results of~\cite{duchi2010dual} on distributed optimization. The proposed algorithms in~\cite{rad2010distributed} and~\cite{jadbabaie2012non}, involve agents that repeatedly receive heterogeneous, private, random i.i.d. signals generated from a global likelihood function and the goal of agents is to learn the true state of the world using local marginals. Both papers  show that under mild assumptions all agents eventually estimate the true parameter correctly. In~\cite{jadbabaie2012non}, agents update their prior beliefs using private observations and then compute a weighted average of their beliefs with that of their neighbors, while in~\cite{rad2010distributed}, agents update the logarithms of their beliefs using the local log-likelihood function. In both cases, under mild assumptions agents eventually learn the true state.  We show that the results of~\cite{rad2010distributed}, have a very interesting optimization-based interpretation.
Exploring this connection and building an optimization-based rationale helps us quantify the pros and cons of different approaches to the problem.

 A key unifying observation that links both recursive Bayesian learning and  Maximum Likelihood Estimation(MLE) problems to online optimization (even  in the centralized setting) is the view of MLE in  Bayesian framework as an optimization in which the inner product of the belief vector and the global log-likelihood function (represented as a vector) is maximized, subject to the belief vector being a probability distribution over the space of parameters. Perhaps less well-known, is the fact that Bayesian parameter estimation can be derived from the exact same setup if the Kullback-Leibler divergence from a prior belief is added to the optimization cost or used as a {\it proximal function}. We show an efficient distributed counterpart of this idea using a stochastic variant of Nesterov's projected dual averaging~\cite{nesterov2009primal}. Aggregating their private log-likelihood functions, agents average their local information, and in the same time step update estimates of the centralized beliefs in a step akin to applying Bayes rule on the aggregated log-likelihoods. When the true state is globally identifiable and the network is connected, we show that agents reach consensus on the beliefs in  probability. More specifically, we prove that with high probability the convergence is exponentially fast with a rate dependent on the {\it average expected discrimination information} for the true state over the second likeliest state captured by the KL divergence of the observations under two aforementioned states. We further show that indeed there is no need for a diminishing stepsize rule as in general subgradient approaches, and a fixed stepsize of 1 can be used. Interestingly, the method recovers the distributed MAP algorithm proposed by~\cite{rad2010distributed} as a special case.

The rest of the paper is organized as follows. In the next section, we introduce the model under which agents interact, define our learning problem and formulate it as a constrained maximization. In section III, we recover Bayesian estimation with dual averaging. In section IV we show applying gossip distributed dual averaging under constant, unit stepsize rules results in learning in the probability sense, and the convergence is exponentially fast. Section V concludes.

\section{Preliminaries}

\subsection{Agents and Observation}

We consider a network consisting of a  finite number of agents $V=\{1,2,\ldots,n\}$. The agents indexed by  $i\in V$ seek a fixed, unique, true state of the world  $\theta^\ast\in \Theta$ with $\Theta=\{\theta_1,\theta_2,\ldots,\theta_m\}$ denoting a finite set of possible states. At each time $t\geq 0$, belief of agent $i$ is denoted by  $\mu_{i,t}(\theta) \in  \Delta \Theta$, where $\Delta \Theta$ is a probability distribution over the set $\Theta$.   In particular, $\mu_{i,0}(\theta) \in \Delta \Theta$ denotes the prior belief of agent $i$ about the states of the world. For each agent $i$, we assume the prior $\mu_{i,0}$ is in the interior of the probability simplex and as a result has no zero elements\footnote{We will see that this assumption is just for dealing with log-likelihood functions and technical issues; otherwise, we only need strict positivity of beliefs over the true state.}.

The learning model is given by a conditional likelihood function $\ell(s^t|\theta_j)$ which is governed by a state of the world $\theta_j \in \Theta$. The signal $s_t=(s_1^t,s_2^t,\ldots,s_n^t)\in S_1\times \dots \times S_n$ is generated at each time $t$, and  $s_i^t\in S_i$ denotes the signal privately observed by agent $i$ at time $t$, where $S_i$ is the signal space for agent $i$. $\ell_i(.|\theta_j)$ represents the $i$-th marginal of $\ell(.|\theta_j)$, and we let the vector $\ell_i(.|\theta)=[\ell_i(.|\theta_1),...,\ell_i(.|\theta_m)]^T$, for any $i\in V$, where $\ell_i(.|\theta_j)>0$ for all signals at all times. 
Agent $i$ at time $t$, has access to the parametrized likelihood of the realized private signal $s_i^t$, i.e., it knows the value of $\ell_i(s_i^t|\theta)$, but does not have access to the likelihood functions of other agents, i.e., it does not know $\ell_j(.|\theta)$ for any $j\neq i$.  Generated signals are i.i.d. over time and also independent over agents. We also define $\bar{\Theta}_i$ as the set of states that are observationally equivalent to $\theta^\ast$ for agent $i$ ; in other words, $\bar{\Theta}_i=\{\theta_j \in \Theta : \ell_i(s_i|\theta_j)= \ell_i(s_i|\theta^\ast) \ \ \forall s_i\in S_i\} $ with probability one. Let $\bar{\Theta}=\cap_{i=1}^n\bar{\Theta}_i$ be the set of states that are observationally equivalent to $\theta^\ast$ from all agents perspective. We assume
\begin{description}
\item[A1.]The true state is globally identifiable, and hence, $\bar{\Theta}=\{\theta^\ast\}$.  
\item[A2.]Each log-marginal $\log \ell_i(.|\theta_j)$ has a bounded variance. 
\end{description}

The probability triple $(\Omega,\mathcal{F},\mathbb{P}^{\theta^\ast})$ is defined such that $\Omega=(\otimes_{i=1}^nS_i)^\mathbb{N}$, $\mathcal{F}_t$ is the smallest $\sigma$-field containing the information about all agents up to time $t$, and $\mathbb{P}^{\theta^\ast}$ is the true probability measure with respect to $\Omega$ with $\mathbb{E}^\ast$ being its corresponding expectation operator. $\mathbb{N}$ represents the natural numbers and $\displaystyle{\mathcal{F}=\cup_{t=1}^{\infty}\mathcal{F}_t}$. 

\begin{definition}
Agent $i\in V$ asymptotically learns the true parameter $\theta^\ast$ on a path $\{s^t\}_{t=1}^{\infty}$ if, along that path,
\begin{align*}
&\mu_{i,t}(\theta^\ast) \rightarrow 1 \ \ \ \ \text{as}   \ \ \ \  t\rightarrow \infty.
\end{align*}
\end{definition}

The definition is intuitive as learning occurs when agents assign probability one to the unique  true parameter.

\subsection{Time Model and Communication Structure}
 The interaction between agents is captured  by an undirected  graph $G=(V,E)$, where $V$ is the set of agents and if there is a link between agent $i$ and agent $j$, the pair $\{i,j\}$ belongs to the set $E$.  We let $\mathcal{N}_i=\left\{j\in V: \{i,j\}\in E\right\}$ be the set of neighbors of agent $i$.
 
Agents communication conforms to an invariant gossip algorithm\cite{boyd2006randomized}, wherein each node has a clock which ticks according to a rate 1 Poisson process. Equivalently, there is a single global clock which ticks according to a rate $n$ Poisson process at times $T_t$, where $\{T_t-T_{t-1}\}$ are i.i.d. exponential random variables with rate $n$. In the analysis, we use the index $t$ to refer to the $t$-th time slot $[T_{t-1},T_t)$, $t\geq 0$. At each tick $T_t$ of the global clock, agent $I_t\in V$ is picked uniformly at random. Then, it contacts a neighbor $J_t\in V$ with probability $P_{I_tJ_t}$, and they update their belief. Denoting the communication matrix by $W(t)$, this amounts to $W(t)$ taking the form
\begin{align*}
W(t)=I-\frac{(\mathbf{e}_{I_t}-\mathbf{e}_{J_t})(\mathbf{e}_{I_t}-\mathbf{e}_{J_t})^T}{2},
\end{align*}
with probability $\frac{1}{n}P_{I_tJ_t}$, where $\mathbf{e}_i$ is the $i$-th unit vector in the standard basis of $\Real^n$. Hence, the matrix $P=[P_{ij}]$ has nonnegative entries and $P_{ij}>0$ only if $\{i,j\}\in E$. By definition, $P$ is row stochastic with largest eigenvalue 1. We assume
\begin{description}
\item[A3.]The network is \textit{connected} i.e., there exists a path from any agent $i$ to any agent $j$, and the second largest eigenvalue of $\mathbb{E}[W(t)]$ is strictly less than one in magnitude. 
 \end{description}
 
The connectivity constraint in assumption (A3) guarantees the information flow in the network. The assumption, for instance, holds if the underlying structure of the network is connected and nonbipartite.
\subsection{Problem Setup and Formulation}
The MLE problem of  finding  the likeliest true state, can  be formulated  in terms of a belief vector $\mu$ as the following optimization
\begin{align}
&\max_{\mu \in \Delta \Theta} \bigg\{ f(\mu)\triangleq \mu^T\sum_{i=1}^n\mathbb{E}^\ast_{s_i}[\log \ell_i(s_i|\theta)] \bigg\} \label{masale3},
\end{align}
with $\mu^\ast$ being its optimal solution. In the next section we discuss that a regularization term can be added to the objective function of \eqref{masale3} or used as a common {\it proximal function} among the agents. Alternatively, one might cast \eqref{masale3} as a quest for finding the MLE solution
\begin{align}
\theta^\ast=\text{argmax}_{\theta_j \in \Theta} \bigg\{\mathbb{E}^\ast_s[\log\ell(s|\theta_j)] \bigg\} \label{masale2}.
\end{align}
The equivalence of \eqref{masale3} and \eqref{masale2} follows immediately from the independence of agents observations, and the global identifiability of $\theta^\ast$ (assumption A1) which guarantees that \eqref{masale3} has a unique maximizer. In the sequel, without loss of generality, we assume the components of the vector $\mathbb{E}^\ast_s\log\ell(s|\theta)$ are in descending order, i.e.
 \begin{align}
\mathbb{E}^\ast_{s}[\log \ell(s|\theta_1)]>\mathbb{E}^\ast_{s}[\log \ell(s|\theta_2)]\geq ... \geq \mathbb{E}^\ast_{s}[\log \ell(s|\theta_m)], \label{orderassum}
\end{align}
where the strict inequality on the left-hand side of \eqref{orderassum} is due to uniqueness of $\theta^\ast=\theta_1$.
Hence, $\theta_1$ is the unique true state that is aimed to be recovered, and $\mu^\ast=\mathbf{e}_1$.

\section{Bayesian Estimation via Nesterov's Dual Averaging}

The learning problem formulated in \eqref{masale3} is a maximization over a closed convex set, so the structure of the problem allows us to apply a distributed generalization of the centralized dual averaging method proposed in~\cite{nesterov2009primal}. First, however, we show how Bayesian learning can be viewed with an optimization lens.

A common approach to tackle problem \eqref{masale3} is to consider the empirical average as the cost function, and solve the online stochastic learning problem. To this end, we employ a regularized dual averaging scheme generating a sequence of iterates $\{\mu_t , z_t\}^{\infty}_{t=0}$, where $\mu_t \in \Delta \Theta$ and $z_t \in \mathbb{R}^m$. At time period $t$ the algorithm receives $g_t$, the stochastic gradient of the objective function, and performs the following set of centralized updates:
\begin{align}
&z_{t+1}=z_t+g_t \ \ \ \ \text{and} \ \ \ \ \mu_{t+1}=\prod_{\Delta \Theta}^{\psi}(z_{t+1},\alpha_t)  \label{standard},
\end{align}
where $\{\alpha_t\}^{\infty}_{t=0}$ is a non-increasing sequence of positive stepsize, $\psi(.)$ is a so called {\it proximal function}, and
\begin{align}
\prod_{\Delta \Theta}^{\psi}(z,\alpha) \triangleq \text{argmin}_{x\in \Delta \Theta}\bigg\{-<z,x>+\frac{1}{\alpha}\psi(x)\bigg\} \label{mini},
\end{align}
with $<z,x>$ being the standard inner product in the space of $\mathbb{R}^m$.

The {\it dual} update $z$, essentially integrates the stochastic gradients, and the second update projects the integration on the feasible set while regularizing the projection using a proximal function.

A particularly relevant example of a proximal function is the Kullback-Leibler (KL) divergence (also known as relative entropy) from an initial belief $\mu_0$ defined as\cite{cover2012elements}:
\begin{align}
\psi(x)=D_{KL}(x||\mu_0)\triangleq \sum_{i=1}^m [x]_i\log \frac{[x]_i}{{\mu_0(\theta_i)}}, \label{KLfunction}
\end{align}
for any $x \in \Delta \Theta$, where $[x]_i$ is the $i$-th component of the vector $x$. It is straightforward to verify that KL divergence from $\mu_0$ is strongly convex with respect to the $\ell_1$-norm on the probability simplex $\{x|x\geq 0 , \sum_{i=1}^m [x]_i=1\}$\footnote{At origin we consider the limit; in other words, we define $0\log(0)=0$.}.
 
The following proposition shows how the set of updates \eqref{standard} equipped with the KL divergence could be viewed as an optimization counterpart of Bayesian rule.

\begin{proposition}\label{centralized}
 Given update rules \eqref{standard} with stepsize sequence $\{\alpha_t=1\}^{\infty}_{t=0}$, using KL divergence as the proximal function, following the stochastic gradient at each time period $t$, and letting $z_0=0$, we obtain the Bayes rule as
\begin{align}
\mu_t(\theta)=\frac{\mu_{t-1}(\theta)\odot \ell(s^t|\theta)}{\sum_{j=1}^m\mu_{t-1}(\theta_j) \ \ell(s^t|\theta_j)} \label{bayesian},
\end{align}
where $\odot$ is component-wise multiplication. 
\end{proposition}

\begin{proof}
To solve \eqref{masale3} with updates \eqref{standard}, since the stochastic gradient is $g_t=\sum_{i=1}^n\log \ell_i(s_i^t|\theta)$, performing the first update, we have
\begin{align*}
z_t=\sum_{i=1}^n\sum_{\tau=0}^{t-1}\log \ell_i(s_i^\tau|\theta)=\sum_{\tau=0}^{t-1}\log \ell(s^\tau|\theta).
\end{align*}
Using $\psi(x)=\sum_{j=1}^m [x]_j \log \frac{[x]_j}{\mu_0(\theta_j)}$ as the proximal function, we need to solve
\begin{align}
\mu_t=\text{argmin}_{x \in \Delta \Theta} \bigg\{- x^T z_t+\sum_{j=1}^m [x]_j \log \frac{[x]_j}{\mu_0(\theta_j)}\bigg\} \label{lagranjnashode}.
\end{align}
Leaving the positivity constraint implicit, we can write \eqref{lagranjnashode} as the maximization of the following Lagrangian {\small
\begin{align}
\mathbb{L}(x,\lambda)= x^T \sum_{\tau=0}^{t-1}\log \ell(s^\tau|\theta)-\sum_{j=1}^m [x]_j \log \frac{[x]_j}{\mu_0(\theta_j)}+\lambda(x^T \mathbf{1}-1) \label{lagranjshode},
\end{align}}
where $\mathbf{1}$ is the vector of all ones. Differentiating \eqref{lagranjshode} we get
{\small
\begin{align*}
&\frac{\partial}{\partial_{[x]_j}}\mathbb{L}(x,\lambda)= \sum_{\tau=0}^{t-1}\log \ell(s^\tau|\theta_j)-\log [x]_j+\log \mu_0(\theta_j)-1+\lambda \\
&\frac{\partial}{\partial_\lambda}\mathbb{L}(x,\lambda)=x^T\mathbf{1}-1.
\end{align*}}
Setting the above equations to zero, we get
\begin{align}
&x=\exp^{\lambda-1}\mu_0\odot \prod_{\tau=0}^{t-1} \ell(s^\tau|\theta) \label{eq9} \\
&x^T\mathbf{1}=1 \label{eq10},
\end{align}
and replacing $x$ in \eqref{eq10} by \eqref{eq9} we have
\begin{align}
\exp^{\lambda-1}=\frac{1}{\sum_{j=1}^m \mu_0(\theta_j)\prod_{\tau=0}^{t-1} \ell(s^\tau|\theta_j) }.\label{eq11}
\end{align}
Hence, by \eqref{eq9} and \eqref{eq11} we have
\begin{align*}
\mu_t(\theta)=\frac{\mu_0(\theta)\odot \prod_{\tau=0}^{t-1}\ell(s^\tau|\theta)}{\sum_{j=1}^m\mu_{0}(\theta_j)\prod_{\tau=0}^{t-1} \ell(s^\tau|\theta_j)},
\end{align*}
and \eqref{bayesian} follows by compositionality of the above equation.
\end{proof}

We derived a closed-form solution for $\mu_t(\theta)$ that essentially performs the Bayesian update; each agent aggregates information up to time $t$, and then, infers the posterior from prior. One can prove the almost sure convergence of $\mu_t(\theta)$ combining the arguments in\cite{blackwell1962merging} and\cite{jadbabaie2012non}. However, we are interested in solving \eqref{masale3}
in a decentralized manner, and we only use a generalized result of Proposition \ref{centralized} later.

\section{Distributed Stochastic Learning}
We now show that the centralized optimization studied in the previous section can be distributed over the network. Contrary to the centralized algorithm, each agent $i\in V$, at $t$-th slot only observes ${g}_{i,t}$, the stochastic gradient  of its associated log-likelihood function, while it does not have access to the signals of other agents. The communication structure is based on a randomized gossip scheme. Let the global Poisson clock at the beginning of the $t$-th slot tick for agent $i$ (with probability $\frac{1}{n}$), and let agent $i$ contact a neighboring node $j$ (with probability $P_{ij}$). Then, agents $i$ and $j$ average their accumulated observations from previous slots, and add their new stochastic gradients to form the following online updates
 {\normalsize
\begin{align}
&z_{i,t+1}=\frac{z_{i,t}+z_{j,t}}{2}+{g}_{i,t} \ \ \text{and} \ \ z_{j,t+1}=\frac{z_{i,t}+z_{j,t}}{2}+{g}_{j,t} , \label{z1}
\end{align}}
while in that slot, any other agent $k\not \in \{i,j\}$ does not contact its neighbors, and only follows its own stochastic gradient ${g}_{k,t}$, so we have
\begin{align}
z_{k,t+1}=z_{k,t}+{g}_{k,t}. \label{z2}
\end{align}
Having updated their observations, all agents calculate their estimates
\begin{align}
\mu_{i,t+1}(\theta)=\prod_{\Delta \Theta}^{\psi}(z_{i,t+1},\alpha_t) \label{distributed},
\end{align}
where $\prod_{\Delta \Theta}^{\psi}(z,\alpha)$ is previously defined in \eqref{mini}. Letting
\[Z_t= \left( \begin{array}{cccc}
z_{1,t} \\
z_{2,t} \\
\vdots \\
z_{n,t}  \end{array} \right)
\ \ \ \text{and} \ \ \ G_t= \left( \begin{array}{cccc}
{g}_{1,t} \\
{g}_{2,t} \\
\vdots \\
{g}_{n,t}  \end{array} \right),\]
the set of updates \eqref{z1} and \eqref{z2} can be represented in the matrix form as follows
\begin{align}
Z_{t+1}=\tilde{W}(t)Z_t+G_t, \label{matrixform}
\end{align}
where $\tilde{W}(t)=W(t)\otimes \mathbf{I}_{m\times m}$, and the random matrix $W(t)$ with probability $\frac{1}{n}P_{ij}$ takes the form
\begin{align}
W(t)=I-\frac{(\mathbf{e}_i-\mathbf{e}_j)(\mathbf{e}_i-\mathbf{e}_j)^T}{2}. \label{matrixW}
\end{align}
We use the above distributed stochastic scheme to optimize \eqref{masale3} equipped with the KL divergence defined in \eqref{KLfunction}. It is noteworthy that in the distributed setting, employing the KL divergence from the initial belief, each agent exhibits inertia to a default opinion over the states. We prove in the next lemma that $\mu_{i,t}(\theta)$ preserves a Bayes-like evolution.
\begin{lemma}\label{mudistr}
Given the set of update rules \eqref{z1}-\eqref{z2}-\eqref{distributed} with stepsize sequence $\{\alpha_t=1\}^{\infty}_{t=0}$, following its stochastic gradient ${g}_{i,t}=\log \ell_i(s_i^t|\theta)$ at $t$-th time period, if we let $z_{i,0}=0$\footnote{Regardless of the zero initial value, all the results hold asymptotically, and this condition only simplifies our derivation.}, agent $i$'s estimator evolves as
\begin{align}
\mu_{i,t}(\theta)=\frac{\mu_{i,0}(\theta)\odot \exp[t \Phi_{i,t}(\theta)]}{\sum_{j=1}^m \mu_{i,0}(\theta_j)\exp [t\Phi_{i,t}(\theta_j)]}  \label{javab},
\end{align}
where  
{\small
\begin{align}
\Phi_{i,t}(\theta)=\frac{1}{t}\sum_{\tau=0}^{t-1}\sum_{k=1}^n  \bigg[\prod_{\rho=1}^{t-1-\tau}W(t-\rho)\bigg]_{ik} \log \ell_k (s_k^\tau|\theta). \label{alireza}
\end{align}}
\end{lemma}

\begin{proof}
The discrete-time linear system \eqref{matrixform} has the closed-form solution
\begin{align*}
&Z_t=\bigg(\prod_{\rho=1}^t \tilde{W}(t-\rho)\bigg)Z_0+\sum_{\tau=0}^{t-1}\bigg(\prod_{\rho=1}^{t-\tau-1}\tilde{W}(t-\rho)\bigg)G_{\tau}.
\end{align*}
Letting $Z_0=0$, since $\tilde{W}(t)=W(t)\otimes \mathbf{I}_{m\times m}$, by basic properties of Kronecker product, we can extract $z_{i,t}$ from $Z_t$ for each $i$ to get
\begin{align*}
z_{i,t}=\sum_{\tau=0}^{t-1}\sum_{k=1}^n  \bigg[\prod_{\rho=1}^{t-1-\tau}W(t-\rho)\bigg]_{ik} \log \ell_k (s_k^\tau|\theta)=t\Phi_{i,t}(\theta).
\end{align*}
We now need to solve \eqref{distributed} to complete the proof. The argument follows in the same fashion as Proposition \ref{centralized}. Forming the Lagrangian as in \eqref{lagranjshode}, writing the first order conditions using $z_{i,t}$ derived above, and following the same steps as in \eqref{eq9}, \eqref{eq10} and \eqref{eq11}, the closed-form solution \eqref{javab} follows immediately.
\end{proof}

Equation \eqref{javab} shows that at each time period $t\geq 0$, the set of distributed update rules \eqref{z1}-\eqref{z2}-\eqref{distributed} construct a {\it Gibbs distribution} over the states. As we shall see in the next subsection, $\Phi_{i,t}(\theta)$ plays a key role on the convergence of the sequence of distributions generated over time.

\subsection{Convergence Analysis}
We now exhibit that aggregating information over time, agents have arbitrarily close opinions in a connected network. This is captured by the fact that the limit of \eqref{alireza} is independent of $i$ which indexes agents. In this direction, we study the limit behavior of \eqref{alireza} in the following lemma.

\begin{lemma}\label{phi}
Under assumptions (A2) and (A3), the vector $\Phi_{i,t}(\theta)$ defined in \eqref{alireza} converges in the probability sense as follows
\begin{align*}
 \Phi_{i,t}(\theta)\overset {p}{\longrightarrow}  \Phi_{\infty}(\theta)=\frac{1}{n}\sum_{k=1}^n\mathbb{E}^\ast_{s_k}[\log \ell_k(s_k|\theta)].
\end{align*}
\end{lemma}

\begin{proof}
The sequence $\{W(t)\}_{t=0}^\infty$ is doubly stochastic, and the product term in $\Phi_{i,t}(\theta)$ preserves doubly stochasticity, so for any $j\in \{1,2,...,m\}$ we have

{\small
\begin{align*}  
\text{var}[\Phi_{i,t}(\theta_j)] &= \frac{1}{t^2}\sum_{\tau=0}^{t-1}\sum_{k=1}^n\bigg[\prod_{\rho=1}^{t-1-\tau}W(t-\rho)\bigg]^2_{ik} \text{var}[\log \ell_k (s_k|\theta_j)]\\
&\leq \frac{1}{t} \sum_{k=1}^n \text{var}[\log \ell_k (s_k|\theta_j)] .
\end{align*}}
Hence, bounded variance assumption (A2) guarantees $\Phi_{i,t}(\theta)-\mathbb{E}[\Phi_{i,t}(\theta)]\overset {p}{\longrightarrow}0$. It can be shown\cite{boyd2006randomized} that 
\begin{align*}
\mathbb{E}[W(t)]=\mathbb{E}[W(0)]=I-\frac{1}{2n}D+\frac{P+P^T}{2n},
\end{align*}
where the diagonal matrix $D$ is of the form $D_i=\sum_{j=1}^n[P_{ij}+P_{ji}]$. The fact that the sequence $\{W(t)\}_{t=0}^\infty$ is i.i.d. and doubly stochastic, and the second largest eigenvalue of $\mathbb{E}[W(t)]$ is less than one in magnitude (A3), entails\cite{tahbaz2010consensus}
\begin{align*}
W(t)W(t-1)...W(1)\longrightarrow \frac{1}{n}\mathbf{1}\mathbf{1}^T,
\end{align*}
almost surely, which results in
{\small
\begin{align*}
\mathbb{E}[\Phi_{i,t}(\theta)]&=\sum_{k=1}^n\frac{1}{t}\sum_{\tau=0}^{t-1}\bigg[\prod_{\rho=1}^{t-1-\tau}W(t-\rho)\bigg]_{ik} \mathbb{E}^\ast_{s_k}[\log \ell_k (s_k|\theta)]\\
&\longrightarrow \frac{1}{n} \sum_{k=1}^n\mathbb{E}^\ast_{s_k}[\log \ell_k (s_k|\theta)],
\end{align*}}
where we used the fact that Ces$\grave{\text{a}}$ro mean preserves the limit.
\end{proof}

There is an interesting connection between the previous lemma and the distributed MAP algorithm proposed in~\cite{rad2010distributed} where authors establish that the point maximizer of $\Phi_{i,t}(\theta)$ over $\Theta$ converges in probability to $\theta^\ast$ in a strongly connected network. However, we still need to demonstrate that the estimator $\mu_{i,t}(\theta)$ is weakly consistent for any agent $i$. To this end, we prove that applying a Bayes-like update in the same time-scale of receiving $z_{i,t}$, the belief vector $\mu_{i,t}(\theta)$ converges in probability to the unique maximizer of \eqref{masale3} which is a Dirac distribution over $\theta^\ast$.  

\begin{theorem}\label{theorem} Given conditions in the Lemmas \ref{mudistr} and \ref{phi}, agent $i$'s estimator is weakly consistent, that is,
\begin{align*}
\mu_{i,t}(\theta) \overset {p}{\longrightarrow} \mu^\ast=\mathbf{e}_1 \ \ \text{as}  \ \ t\rightarrow \infty.\\
\end{align*}
\end{theorem}
\begin{proof}
We have the explicit form of the estimator $\mu_{i,t}$ according to Lemma \ref{mudistr}. Therefore,
\begin{align*}
\mu_{i,t}&(\theta^\ast=\theta_1)=\frac{\mu_{i,0}(\theta_1)\exp[t \Phi_{i,t}(\theta_1)]}{\sum_{j=1}^m \mu_{i,0}(\theta_j)\exp [t\Phi_{i,t}(\theta_j)]} \\
&=\bigg(1+\sum_{j\geq 2}\frac{ \mu_{i,0}(\theta_j)}{\mu_{i,0}(\theta_1)}\exp [t\Phi_{i,t}(\theta_j)-t\Phi_{i,t}(\theta_1)]\bigg)^{-1}.
\end{align*}
Under purview of Lemma \ref{phi} and equation \eqref{orderassum}, $[\Phi_{i,t}(\theta_j)-\Phi_{i,t}(\theta_1)]$ converges to a negative number for any $j \geq 2$, and hence $\mu_{i,t}(\theta_1)\overset{p}{\longrightarrow}1$. The fact that $\mu_{i,t}(\theta) \in \Delta \Theta$ implies that $\mu_{i,t}(\theta_j)\overset{p}{\longrightarrow}0$ for all $j\geq 2$, so $\mu_{i,t}(\theta)\overset{p}{\longrightarrow}\mu^\ast=\mathbf{e}_1$.
\end{proof}

Theorem \ref{theorem} too underscores the trade-off between the adaptation and learning in the network. In many distributed optimization settings the stepsize sequence must vanish to allow nodes to reach consensus. However, the result of Theorem \ref{theorem} holds for unit stepsize sequence which guarantees learning even under continuous information injection to the network. This stems from the fact that the algorithm allows $z_{i,t}$ to grow unboundedly in each direction, while it lets the true state to be the influential component by having the largest exponential rate in the generated Gibbs distribution.
\subsection{Learning Rate Analysis}
In this section we characterize the convergence rate of the estimator $\mu_{i,t}(\theta)$. More specifically, we prove that convergence occurs exponentially fast with a rate dependent on the {\it average expected discrimination information} for $\theta_1=\theta^\ast$ over $\theta_2$, where $\theta_2$ is the state with the second largest expected log-likelihood \eqref{orderassum}.
\begin{definition}
The expected discrimination information of agent $i$ for $\theta_1=\theta^\ast$ over any $\theta_j$ is 
\begin{align*}
D_{KL}\bigg(\ell_i(.|\theta_1)||\ell_i(.|\theta_j)\bigg)=\mathbb{E}^\ast_{s_i}\bigg[\log \frac {\ell_i (s_i|\theta_1)}{\ell_i (s_i|\theta_j)}\bigg].
\end{align*}
Denoting by $D(\theta_j)$, the {\it average expected discrimination information} for $\theta_1=\theta^\ast$ over $\theta_j$ is defined as
\begin{align}
D(\theta_j)& \triangleq \frac{1}{n}\sum_{i=1}^nD_{KL}\bigg(\ell_i(.|\theta_1)||\ell_i(.|\theta_j)\bigg). \label{}
\end{align} 
\end{definition}
As an immediate consequence of the definition above, one can see from Lemma \ref{phi} that $D(\theta_j)=\Phi_{\infty}(\theta_1)-\Phi_{\infty}(\theta_j)$ for any $j \geq 1$, and $D(\theta_1)=0$.
\begin{theorem}\label{rate}
Given conditions in the Lemmas \ref{mudistr} and \ref{phi}, for any $\epsilon>0$ and $t$ large enough, the estimator $\mu_{i,t}(\theta_1)$ can be bounded as
\begin{align}
\big|\mu_{i,t}(\theta_1)-1\big| \leq \mathcal{K} \exp[(-D(\theta_2)+\epsilon)t], \label{rate11}
\end{align}
with probability at least $1-\delta(\epsilon,t)$, where $\mathcal{K}$ is a constant. 
\end{theorem}

\begin{proof}
Following the lines in the proof of Theorem \ref{theorem}, we have
\begin{align*}
\mu_{i,t}(\theta_1)&=\bigg(1+\sum_{j\geq 2}\frac{ \mu_{i,0}(\theta_j)}{\mu_{i,0}(\theta_1)}\exp [t\Phi_{i,t}(\theta_j)-t\Phi_{i,t}(\theta_1)]\bigg)^{-1}\\
&\geq 1-\sum_{j\geq 2}\frac{ \mu_{i,0}(\theta_j)}{\mu_{i,0}(\theta_1)}\exp [t\Phi_{i,t}(\theta_j)-t\Phi_{i,t}(\theta_1)],
\end{align*}
where in the last step we used the inequality
\begin{align*}
1-\lambda \leq (1+\lambda)^{-1} \ \ \ \forall \lambda \geq 0.
\end{align*}
Letting $b_j \triangleq \Phi_{i,t}(\theta_j)-\Phi_{\infty}(\theta_j)$, we derive 
\begin{align*}
\big|\mu_{i,t}(\theta_1)&-1\big|\leq \sum_{j\geq 2}\frac{ \mu_{i,0}(\theta_j)}{\mu_{i,0}(\theta_1)}\exp [t\Phi_{i,t}(\theta_j)-t\Phi_{i,t}(\theta_1)]\\
&\leq \max_k \frac{ \mu_{i,0}(\theta_k)}{\mu_{i,0}(\theta_1)} \sum_{j\geq 2}\exp [t\Phi_{i,t}(\theta_j)-t\Phi_{i,t}(\theta_1)]\\
&= \max_k \frac{ \mu_{i,0}(\theta_k)}{\mu_{i,0}(\theta_1)} \sum_{j\geq 2}\exp [(-D(\theta_j)+b_j-b_1)t].
\end{align*}

One can see in the proof of Lemma \ref{phi} that $\text{var}[\Phi_{i,t}(\theta_j)]$ decays with a rate $C/t$ for some constant $C>0$; hence, for any $\epsilon>0$ and $j \geq 1$, by Chebyshev's inequality we obtain
\begin{align*}
\mathbb{P}(|b_j|\geq \epsilon)\leq \frac{C}{\epsilon^2t}.
\end{align*}
Combining with $D(\theta_m)\geq...\geq D(\theta_2)>D(\theta_1)=0$ by \eqref{orderassum}, for any $\epsilon>0$ and $t$ large enough, we have
\begin{align*}
\big|\mu_{i,t}(\theta_1)-1\big|& \leq \max_k \frac{ \mu_{i,0}(\theta_k)}{\mu_{i,0}(\theta_1)} \sum_{j\geq 2}\exp [(-D(\theta_2)+2\epsilon)t]\\
&=(m-1)\max_k \frac{ \mu_{i,0}(\theta_k)}{\mu_{i,0}(\theta_1)}\exp [(-D(\theta_2)+2\epsilon)t],
\end{align*}
with probability at least $1-\frac{C}{\epsilon^2t}$. Hence, the constants in \eqref{rate11} are determined as
\begin{align*}
\mathcal{K}=(m-1)\max_k \frac{ \mu_{i,0}(\theta_k)}{\mu_{i,0}(\theta_1)} \ \ \text{and} \ \ \delta(\epsilon,t)=\frac{4C}{\epsilon^2t},
\end{align*}
and we are done.
\end{proof}

Theorem \ref{rate} suggests that the proposed distributed stochastic learning method in \eqref{z1}-\eqref{z2}-\eqref{distributed} converges exponentially fast with high probability. Moreover, agents learn the true state with a rate dependent on the KL divergence of observations under the true state from observations under the second likeliest state. This, indeed, stresses the efficiency of the algorithm.

\section{Conclusion}
We studied a distributed parameter estimation problem over networks 
when agents receive a sequence of i.i.d. signals but the signals are not
informative enough to identify the true parameter. Using a randomized, gossip dual averaging, agents aggregate local log-likelihood functions, 
and then perform a Bayes-like update on the averaged information to collectively recover the truth. Assuming connectivity of the network and global identifiability of the true
state, we showed that agents beliefs reach consensus and collapse to a degenerate distribution over the true parameter, and with high probability the convergence is exponentially fast. We also proved that the rate of exponential depends on the KL divergence of observations under true state from observations under second likeliest state.
As a salient feature of the algorithm, we showed that contrary to other stochastic
gradient descent methods, the stepsize can be chosen to be fixed and
set to 1. Future directions include addition of dynamics to the parameter
and relaxing the independence conditions on observations as well as specialization to the Gaussian case, where one only needs to update the mean and variance.



\section*{Acknowledgments}
The authors would like to thank Robin Pemantle for many helpful comments and discussions.


{\tiny

\bibliographystyle{IEEEtran}
\bibliography{IEEEabrv,shahin}}

\begin{thebibliography}{10}
\providecommand{\url}[1]{#1}
\csname url@samestyle\endcsname
\providecommand{\newblock}{\relax}
\providecommand{\bibinfo}[2]{#2}
\providecommand{\BIBentrySTDinterwordspacing}{\spaceskip=0pt\relax}
\providecommand{\BIBentryALTinterwordstretchfactor}{4}
\providecommand{\BIBentryALTinterwordspacing}{\spaceskip=\fontdimen2\font plus
\BIBentryALTinterwordstretchfactor\fontdimen3\font minus
  \fontdimen4\font\relax}
\providecommand{\BIBforeignlanguage}[2]{{%
\expandafter\ifx\csname l@#1\endcsname\relax
\typeout{** WARNING: IEEEtran.bst: No hyphenation pattern has been}%
\typeout{** loaded for the language `#1'. Using the pattern for}%
\typeout{** the default language instead.}%
\else
\language=\csname l@#1\endcsname
\fi
#2}}
\providecommand{\BIBdecl}{\relax}
\BIBdecl

\bibitem{BorVar82}
V.~Borkar and P.~Varaiya, ``Asymptotic agreement in distributed estimation,''
  \emph{IEEE Transactions on Automatic Control}, vol. 27, no. 3, pp. 650--655,
  1982.

\bibitem{Tsit88}
J.~N. Tsitsiklis, ``Decentralized detection by a large number of sensors,''
  \emph{Mathematics of Control, Signals, and Systems}, vol. 1, no. 2, pp.
  167--182, 1988.

\bibitem{tsitsiklis1993decentralized}
J.~N. Tsitsiklis \emph{et~al.}, ``Decentralized detection,'' \emph{Advances in
  Statistical Signal Processing}, vol.~2, pp. 297--344, 1993.

\bibitem{mossel2010efficient}
E.~Mossel and O.~Tamuz, ``Efficient bayesian learning in social networks with
  gaussian estimators,'' \emph{Arxiv preprint arXiv:1002.0747}, 2010.

\bibitem{KhanJad2010}
U.A.Khan, S.~Kar, A.~Jadbabaie, and J.~Moura, ``On connectivity, observability,
  and stability in distributed estimation,'' in \emph{49th IEEE Conference on
  Decision and Control (CDC)}.\hskip 1em plus 0.5em minus 0.4em\relax IEEE,
  2010, pp. 6639--6644.

\bibitem{kar2012distributed}
S.~Kar, J.~Moura, and K.~Ramanan, ``Distributed parameter estimation in sensor
  networks: Nonlinear observation models and imperfect communication,''
  \emph{IEEE Transactions on Information Theory}, vol. 58, no. 6, pp.
  3575--3605, 2012.

\bibitem{rad2010distributed}
K.~Rad and A.~Tahbaz-Salehi, ``Distributed parameter estimation in networks,''
  in \emph{49th IEEE Conference on Decision and Control (CDC)}.\hskip 1em plus
  0.5em minus 0.4em\relax IEEE, 2010, pp. 5050--5055.

\bibitem{jadbabaie2012non}
A.~Jadbabaie, P.~Molavi, A.~Sandroni, and A.~Tahbaz-Salehi, ``Non-bayesian
  social learning,'' \emph{Games and Economic Behavior}, vol.~76, no.~1, pp.
  210--225, 2012.

\bibitem{dekel2012optimal}
O.~Dekel, R.~Gilad-Bachrach, O.~Shamir, and L.~Xiao, ``Optimal distributed
  online prediction using mini-batches,'' \emph{The Journal of Machine Learning
  Research}, vol.~13, pp. 165--202, 2012.

\bibitem{tsitsiklis1984problems}
J.~Tsitsiklis, ``Problems in decentralized decision making and computation.''
  DTIC Document, Tech. Rep., 1984.

\bibitem{jadbabaie2003coordination}
A.~Jadbabaie, J.~Lin, and A.~Morse, ``Coordination of groups of mobile
  autonomous agents using nearest neighbor rules,'' \emph{IEEE Transactions on
  Automatic Control}, vol. 48, no. 6, pp. 988--1001, 2003.

\bibitem{EgerMesBook}
M.~Mesbahi and M.~M.~Egerstedt, \emph{Graph theoretic methods in multiagent
  networks}.\hskip 1em plus 0.5em minus 0.4em\relax Princeton Univ Press, 2010.

\bibitem{bullo2009distributed}
F.~Bullo, J.~Cort{\'e}s, and S.~Mart{\'\i}nez, \emph{Distributed control of
  robotic networks: a mathematical approach to motion coordination
  algorithms}.\hskip 1em plus 0.5em minus 0.4em\relax Princeton Univ Pr, 2009.

\bibitem{Olfati05}
R.~Olfati-Saber and J.~Shamma, ``Consensus filters for sensor networks and
  distributed sensor fusion,'' in \emph{44th IEEE Conference on Decision and
  Control}, Seville, Spain, Dec. 2005, pp. 6698 -- 6703.

\bibitem{nedic2009distributed}
A.~Nedic and A.~Ozdaglar, ``Distributed subgradient methods for multi-agent
  optimization,'' \emph{IEEE Transactions on Automatic Control}, vol. 54, no.
  1, pp. 48--61, 2009.

\bibitem{lobel2008distributed}
I.~Lobel and A.~Ozdaglar, ``Distributed subgradient methods over random
  networks,'' in \emph{Proc. Allerton Conf. Commun., Control, Comput}, 2008.

\bibitem{ram2010distributed}
S.~Ram, A.~Nedic, and V.~Veeravalli, ``Distributed stochastic subgradient
  projection algorithms for convex optimization,'' \emph{Journal of
  optimization theory and applications}, vol. 147, no. 3, pp. 516--545, 2010.

\bibitem{nedic20092distributed}
A.~Nedic, A.~Olshevsky, A.~Ozdaglar, and J.~Tsitsiklis, ``On distributed
  averaging algorithms and quantization effects,'' \emph{IEEE Transactions on
  Automatic Control}, vol. 54, no. 11, pp. 2506--2517, 2009.

\bibitem{lopes2007incremental}
C.~Lopes and A.~Sayed, ``Incremental adaptive strategies over distributed
  networks,'' \emph{IEEE Transactions on Signal Processing}, vol. 55, no. 8,
  pp. 4064--4077, 2007.

\bibitem{duchi2010dual}
J.~Duchi, A.~Agarwal, and M.~Wainwright, ``Dual averaging for distributed
  optimization: convergence analysis and network scaling,'' \emph{IEEE
  Transactions on Automatic Control}, pp. 592--607, March 2012.

\bibitem{nesterov2009primal}
Y.~Nesterov, ``Primal-dual subgradient methods for convex problems,''
  \emph{Mathematical programming}, vol. 120, no. 1, 2009.

\bibitem{boyd2006randomized}
S.~Boyd, A.~Ghosh, B.~Prabhakar, and D.~Shah, ``Randomized gossip algorithms,''
  \emph{IEEE Transactions on Information Theory}, vol. 52, no. 6, pp.
  2508--2530, 2006.

\bibitem{cover2012elements}
T.~M. Cover and J.~A. Thomas, \emph{Elements of information theory}.\hskip 1em
  plus 0.5em minus 0.4em\relax John Wiley \& Sons, 2012.

\bibitem{blackwell1962merging}
D.~Blackwell and L.~Dubins, ``Merging of opinions with increasing
  information,'' \emph{The Annals of Mathematical Statistics}, vol.~33, no.~3,
  pp. 882--886, 1962.

\bibitem{tahbaz2010consensus}
A.~Tahbaz-Salehi and A.~Jadbabaie, ``Consensus over ergodic stationary graph
  processes,'' \emph{IEEE Transactions on Automatic Control}, vol. 55, no. 1,
  pp. 225--230, 2010.

\end{thebibliography}

\end{document}